\documentclass[11pt]{amsart}
\usepackage[active]{srcltx}
\usepackage{mathrsfs}
\usepackage[T1]{fontenc}
\usepackage{tcolorbox}

\usepackage{latexsym,enumerate}
\usepackage{amsmath,amssymb,amsthm,amsfonts,latexsym}
\usepackage{graphicx}
\usepackage[bookmarks]{hyperref}

\hypersetup{backref, colorlinks=true}
\hypersetup{backref, colorlinks=true, colorlinks   = true,
	urlcolor     = blue, linkcolor = blue, citecolor   = magenta
}
\def\adh#1{\overline{#1}}

\let\=\partial

\newtheorem {pro}{Proposition}[section]
\newtheorem {thm}[pro]{Theorem}%[section]
\theoremstyle{definition}
 \newtheorem {rem}[pro]{Remark}%[section]
\newtheorem {dfn}[pro]{Definition}%[section]

\newtheorem {exa}[pro]{Example}

\newcommand{\xca}{\mathcal{X}}

\newcommand{\tra}{\mathbf{tr}}

\newcommand{\zca}{\mathcal{Z}}

\newcommand{\R}{\mathbb{R}}

\newcommand{\N}{\mathbb{N}}
\newcommand{\cc}{\mathscr{C}}

\newcommand{\et}{\quad \mbox{and} \quad }

\newcommand{\hn}{\mathcal{H}}

\newcommand{\fbf}{\mathbf{f}}

\newcommand{\omd}{{\Omega}}

\newcommand{\wbf}{\mathbf{w}}\newcommand{\ybf}{\mathbf{y}}

\newcommand{\F}{\mathcal{F}}
\newcommand{\ep}{\varepsilon}

\newcommand{\pa}{\partial}

\newcommand{\hh}{\mathcal{V}}

\newcommand{\bou}{\mathbf{B}}
\newcommand{\sph}{\mathbf{S}}

\newcommand{\supp}{\mbox{\rm supp}}
\newcommand{\xo}{{x_0}}
\title[]{ Semi-supervised learning on singular data sets}

\makeatletter
 \thanks{Research partially supported
by the NCN grant  2021/43/B/ST1/02359.}
\@namedef{subjclassname@2020}{%
\textup{2020} Mathematics Subject Classification}
\makeatletter

\@addtoreset{equation}{section}

\makeatother

\author[ G. Valette]{ Guillaume Valette}

\address[G. Valette]{Instytut Matematyki Uniwersytetu
Jagiello\'nskiego, ul. S. \L ojasiewicza 6, Krak\'ow, Poland}\email{guillaume.valette@im.uj.edu.pl}

\keywords{Machine learning, Semi-supervised learning, singular data cloud, subanalytic set, Sobolev spaces}

\subjclass[2020]{62J07, 68T05, 68Q32, 46E35, 32B20, 14P10}

\begin{document} \maketitle
	\begin{abstract}
	We propose to model singular data clouds by subanalytic sets, and develop an approach to semi-supervised learning in this setting. It relies on a Morrey-type embedding for Sobolev spaces relative to globally subanalytic measures with possibly singular support.
As an illustration, we present a greedy method based on a Monte Carlo approximation of the Sobolev norm of  the candidate solutions. We report numerical experiments on point clouds sampled from real algebraic varieties, as well as a test on a classification problem based on MNIST images  downsampled  into $7\times7$.
		
		% They include all the sets defined by finitely many inequalities on analytic functions.  the globally subanalytic category  are indeed omnipresent in machine learnig problem
		\end{abstract}
		
	\section{Introduction}
	The main task of machine learning is to determine how the available information on a sample can be reasonably generalized to the entire cloud \cite{bishop, survey1, survey2}.  This issue is particularly delicate when labeled data are sparse, and 
		  semi-supervised learning  can then prove a relevant strategy to overcome the difficulties \cite{bns, towards}. Unlabeled data are then  useful to carry information  toward the cloud and working intrinsically  reduces the dimension of the problems.% by working intrinsically. 
	
	Semi-supervised learning algorithms often require  the data to be supported  on a smooth manifold or on an open set with Lipschitz or smooth boundary \cite{semisupermanif,slepthor}, although it is  hard to account why the support of the distribution of the cloud should always be so. %why a cloud should always be distributed with respect to a measure $\mu$ supported on such a set
	 Many authors also assume in addition that the density of the distribution measure is bounded away from zero and infinity. Data are often mapped into feature spaces in order to reduce the ambient dimension, and the image can be singular even if the feature map is polynomial.  The  image of the cloud is then distributed with respect to the push-forward of the distribution measure of the given cloud under the feature map. Its density can thus tend to zero or infinity, even if the feature map is analytic or polynomial.

	In this article, we suggest  globally subanalytic sets and measures as relevant models for singular clouds.  The globally subanalytic category, which includes all the compact sets defined by finitely many equalities on analytic functions, is stable under analytic mappings (restricted to compact sets).          
  Subanalytic sets may admit very general singularities within their closure and a large number of  feature maps involved in machine learning, e.g. polynomial, Gaussian, Wendland, trigonometric, hyperbolic trigonometric, are  subanalytic, which makes the subanalytic sets and measures attractive to modelize  clouds in feature spaces.

	The main motivation for authors to rely on the so-called manifold hypothesis is generally to perform functional analysis. Globally subanalytic sets admit stratifications and cell decompositions, as well as many finiteness properties. Moreover, although these sets admit non metrically conical singularities,  several  results of analysis, such as Poincar\'e inequalities or Sobolev embedding \cite{poincwirt, poincfried, trace, lprime, platrub}, were proved to generalize to these singular sets. Our present idea is to rely on these results so as to produce machine learning algorithms for clouds that are supported on such a singular set.
	
	In order to make it more precise, let us recall
 Corollary $3.5$ of \cite{platrub}, which provided a Morrey-type embedding in this framework:% (we refer to section \ref{sect_math} for  definitions and notations):
	\begin{thm}\label{thm_mor}
		Let $\omd$ be a bounded  open globally subanalytic subset  of $\R^n$ and $\Phi:(A,\mu)\to \omd$  a  globally subanalytic mapping, with $(A,\mu)$ globally subanalytic measured space. Given $p\in [1,\infty)$, there is $k\in \N$ for which we have the  natural continuous embedding
		\begin{equation}\label{eq_embed}
			W_ {\Phi_*\mu}^{k,p}(\omd)\hookrightarrow \check{\cc}^{0,1}(\supp_\omd \, \Phi_*\mu). \end{equation}
	\end{thm}

	Here,   $W^{k,p}_\mu(\omd) $ is the Sobolev space of $\omd$ relative to $\mu$ and $\check{\cc}^{0,1}(\supp_\omd \, \Phi_*\mu)$ denotes the space of functions that are Lipschitz with respect to the inner metric on the support (in $\Omega$) of the pushed-forward measure $ \Phi_*\mu$ (we refer to section \ref{sect_math} for all notations and definitions).

		Let now $\xca:=\{x_i\}_{i\in \N}$ be a sequence in a bounded globally subanalytic open subset $\Omega$ of $\R^n$, and assume  the  $x_i$'s to be drawn randomly with respect to a globally subanalytic probability measure $\mu$ (or a push-forward of such a measure).
	 The above theorem yields that for each $C>0$ the problem 
	\begin{equation*}
	\underset{ ||u||_{W^{k,2}_\mu(\omd)}\le C }{\mbox{argmin}} \;\;\sum_{i=1}^\ell|y_i- u(x_i)|^2 
	\end{equation*}
	has at least one solution in  $W^{k,2}_\mu(\omd) $, for each set of labels $y_1,\dots, y_\ell$. As often, we will   rather solve a ridge regression model that combines the MSE train and the Sobolev norm of the heuristic. We will work with a weighted Sobolev norm, which is equivalent but provides a more relevant minimizer (see (\ref{eq_ridge_theoretical})). The weights actually play the role of regularization parameters.  One advantage of estimating the Sobolev norm relative to the distribution of the sample is that it can be estimated via the law of large numbers by averaging the partial derivatives of $u$ on a random set of data, possibly unlabeled.

 Trying to find or approximate the minimizer in the entire Sobolev space seems to be very difficult, if not untractable. We will thus, as usual in machine learning, rather seek it in a finite dimensional subspace spanned by some chosen hypothesis functions $f_1,\dots, f_b$. One feature of our parametric semi-supervised approach is  however that the norm is independent of the chosen basis. This allows to try different bases in the same algorithm, for instance setting the $f_i$ to be Gaussians centered at points of the cloud, and  reducing successively the radius. We adopted a greedy approach, solving the ridge problem at every level, computing then the residual on the labeled data, and   applying the same method to this residual with a new basis of hypothesis functions (constituted by Gaussians of smaller radius, see Section \ref{sect_algo} for more). Other bases of functions were tried (polynomial, Wendland), and it was tested to combine them in a second phase, which proved to provide robustness, but we only present in Section \ref{sec_tests} tests on our greedy algorithm using Gaussians. It seems that more efforts to select a dictionary of hypothesis functions related to the residual could lead to improvements of the MSE. We present the results on a regression problem on  algebraic subsets of $\R^4$ and then to a classification problem 3,5 vs 8 on MNIST data downsampled into $7\times7$.    %Implementations in Python were realized by   Claude (anthropic), GPT (openAI), and Grok (X). Theorem \ref{thm_sph} is a suggestion of Claude. The rest of the content is due to the author. 

 The above theorem provides no information on the number of derivatives which are needed to ensure  well-posedness. 	
 It is a difficult issue but since this number is   related to the geometry of the cloud and can be chosen independently of the target, it seems  possible to estimate it empirically on target examples.
 	
 The subanalycity assumption can apparently be significantly relaxed, and it seems to be enough to have a measure which is locally (at closure points) bi-Lipschitz equivalent to such a measure (see Remark \ref{rem_equivalence}).
  It is also worthy of notice that the needed mathematical results actually extend to the larger framework of polynomially bounded o-minimal structures expanding the real field \cite{cos,vdd} (see Remark \ref{rem_definissabilite}). 
 Polynomial boundedness is however essential, and for example the results of \cite{lprime, platrub} are no longer valid on an exponential cusp $\{(x,y)\in (0,1)^2:y<e^{-1/x^2}\}$ (endowed with the Lebesgue measure).
 
% It applies to any measure $\nu$ on a locally closed subset $Y$ of $\R^n$ for which we can find a bi-Lipschitz homeomorphism $h:(\omd',\mu')\to (\omd,\mu)$, with $\omd$ bounded globally subanalytic open set and $\mu$ globally subanalytic  measure (or push-forward of such a measure), with $h_*\nu\sim \mu$ . 
% We say locally bi-Lipschitz equivalent at a point $x$ of the closure of the support if $\mu{|}V_x$ is bi-Lipschitz equivalent if this is true in a neighborhood of this point. 
%   to have a locally closed subset of $\R^n$ supported on a set that is locally bi-Lipschitz homeomorphic (locally at every point of $\adh \omd$) to a locally closed bounded subanalytic set. 

 The article is organized as follows. Section \ref{sect_math} recalls some mathematical results that are useful for our purpose. Section \ref{sect_algo} gives our greedy  algorithm, derives consistency from  results of the  preceding section, and establishes statistical bounds relative to the convergence of our Monte Carlo approximation. Section \ref{sec_tests} presents the results of numerical experiments. \section{Analysis on globally subanalytic sets}\label{sect_math}

 	A subset $E\subset \R^n$ is called {\bf semi-analytic} if it is {\it locally}
 	defined by finitely many real analytic equalities and inequalities. Namely, for each $a \in   \R^n$, there is
 	a neighborhood $U$ of $a$ in $\R^n$, and real analytic  functions $f_{ij}, g_{ij}$ on $U$, where $i = 1, \dots, r, j = 1, \dots , s_i$, such that
 	\begin{equation}\label{eq_definition_semi}
 		E \cap   U = \bigcup _{i=1}^r\bigcap _{j=1} ^{s_i} \{x \in U : g_{ij}(x) > 0 \mbox{ and } f_{ij}(x) = 0\}.
 	\end{equation}

 	 The flaw of the  semi-analytic category is that  it is not preserved by analytic morphisms, even when they are proper. To overcome this problem, we prefer working with the  subanalytic sets, which are defined as the projections of the semi-analytic sets.
 	 
 	A subset $E\subset \R^n$  is  {\bf  subanalytic} if 
 	each point $x\in\R^n$ has a neighborhood $U$ such that $U\cap E$ is the image under the canonical projection $\pi:\R^n\times\R^k\to\R^n$ of some relatively compact semi-analytic subset of $\R^n\times\R^k$ (where $k$ depends on $x$).

 	A subset $Z$ of $\R^n$ is  {\bf globally subanalytic} if $\hh_n(Z)$ is a subanalytic subset of $\R^n$, where $\hh_n : \R^n  \to (-1,1) ^n$ is the homeomorphism defined by \begin{equation}\label{eq_nu}
 		\hh_n(x_1, \dots, x_n) :=  (\frac{x_1}{\sqrt{1+|x|^2}},\dots, \frac{x_n}{\sqrt{1+|x|^2}} ).
 	\end{equation}

 	We say that {\bf a mapping (or a function) $f:A \to B$ is globally  subanalytic}, $A \subset \R^n$ and $B\subset \R^m$   globally subanalytic, if its graph is a globally  subanalytic   subset of $\R^{n+m}$.
 	 We will allow the density of measures to be infinite, considering that $f:A\to [0,\infty]$ is globally subanalytic if $f^{-1}(\infty)$ is a globally subanalytic set and the restriction of $f$ to $f^{-1}([0,\infty))$ is a  globally subanalytic function.
  	
  For example, the restriction of an analytic function to a compact subanalytic subset of its domain is
   globally subanalytic. The functions  $\sin 1/x$  and  $e^{-1/x^2}$, $x\in (0,1)$, however fail to be   subanalytic.  Excluding infinite oscillations and flat functions is  the main feature of  the  globally
  subanalytic category and polynomially bounded o-minimal structures in general.
% \end{dfn}
 \subsection{Subanalytic measures}\label{sect_sub_measure}
 Given an $m$-dimensional globally subanalytic  manifold $M$, we denote by $\mu_M$ the measure provided by unit differential $m$-forms, the tangent spaces of  $M$ being endowed with the euclidean metric. This amounts to define $\mu_M$ as the $m$-dimensional Hausdorff measure (restricted to $M$). We    (abusively)   consider $\mu_M$ as defined on the Borel $\sigma$-algebra of a globally subanalytic set  containing $ M$. %To avoid repetitions, we  abbreviate $\mu_M(M)$ into $|M|$.
 
 Given a  $\mu_M$-measurable function $f:M\to [0,\infty]$, we   denote by $f\cdot \mu_M$ the measure $B\mapsto \int_B f(x) d\mu_M(x)$.
 \begin{dfn}\label{dfn_sub_meas}
 	A measure $\mu$ on (the Borel $\sigma$-algebra of) a globally subanalytic set $A$ is said to be globally subanalytic if it is a finite sum of  type $\sum_{i=1}^k \xi_i \cdot \mu_{E_i}$ where, for every $i$, $E_i\subset A$ is a globally subanalytic manifold and   $\xi_i:E_i\to [0,\infty]$ is a globally subanalytic function. We then say that $(A,\mu)$ is a {\bf globally subanalytic measured space}.
 \end{dfn}
 We impose the $E_i$ to be manifolds for convenience. Since globally subanalytic sets can always be decomposed into finitely many globally subanalytic manifolds, this is no loss of generality, and the measure $B\mapsto \sum_{i=1}^k\int_{D_i\cap B} \xi_i \,d \hn^{l_i}$, $l_i=\dim D_i$ ($\hn^{l_i}$ being the Hausdorff measure), with $D_i$ arbitrary globally subanalytic subset of $A$ and $\xi_i:D_i\to [0,\infty]$ globally subanalytic  function  for each $i$, is also globally subanalytic.

 Given a Borel mapping $\Phi:(A,\mu)\to X$, we denote by $\Phi_*\mu$ the push-forward  of $\mu$ under $\Phi$, i.e. $\Phi_*\mu(B):=\mu (\Phi^{-1}(B))$, if $B\subset X$ is a Borel set. Push-forwards of globally subanalytic measures under globally subanalytic mappings may fail to have globally subanalytic density. However, we proved in \cite{platrub}: 
 \begin{pro}\label{pro_push}
 	Let $(A,\mu)$ be a globally subanalytic measured space and let $\Phi:A\to X$ be globally subanalytic. There is a finite collection $\F$  of globally subanalytic manifolds included in $X$ such that
 	$$ \Phi_* \mu =\sum_{M\in \F} \fbf_M \cdot \mu_M, \quad \mbox{with}\quad \fbf_M(y):=\sum_{S\in \tilde{\F}_M} \int_{\Phi_S^{-1}(y)}\zeta_S(x) \,d\mu_{\Phi_S^{-1}(y)}(x), $$
 	where $\tilde{\F}_M$ is for each $M\in \F$ a finite collection of globally subanalytic manifolds comprised in $\Phi^{-1}(M)$ on each of which $\Phi$ induces a $\cc^\infty$ submersion $\Phi_S :S \to M$ and the $\zeta_S:S \to [0,\infty]$ are globally subanalytic functions.
 \end{pro}
 It was derived from this proposition that push-forwards of globally subanalytic measures always have a globally subanalytic support and that their density cannot be a flat function (see \cite[Proposition $2.5$]{platrub}).

 \subsection{Sobolev spaces of globally subanalytic sets.}  We denote by	$|x|$ the    euclidean norm of $x\in \R^n$ and  by	$\adh A$ the closure of a set $A\subset \R^n$. 
 
 Let $\mu$ be a globally subanalytic measure or the push-forward of such a measure on a bounded open globally subanalytic subset $\omd$ of $\R^n$.  We denote by  $X\subset \omd$ the support of $ \mu$.
 
 	Given   $p\in [1,\infty)$, we denote by $L^p_\mu(\omd)$ the Lebesgue space relative to $\mu$, and we set for $k\in \N$
 $$\cc^{k,p}_\mu(\omd):= \{u\in \cc^{k}(\omd) :  \forall \alpha\in \N^n \mbox{ satisfying } |\alpha| \le k,  \frac{\pa^{|\alpha|}u}{\pa x^\alpha} \in L^p_\mu(\omd)\}.$$
 We then denote by  $W^{k,p}_\mu(\omd)$ the  completion of   $\cc^{k,p}_\mu(\omd)$ equipped with
 $$ ||u||_{W^{k,p}_\mu(\omd) }:= \sum_{i\le k} || D^i u||_{L^p _\mu (\omd)}, \qquad \mbox{ where }\quad D^i u:= \left(\frac{\pa^{|\alpha|}u}{\pa x^\alpha} \right)_{|\alpha| =i},$$
 identifying functions whose derivatives of order $\le k$ match $\mu$-almost everywhere.

 If $x$ and $y$ are two points of the same connected component of   $X$, we set
 \begin{equation}\label{eq_inner_dist}
 	d_X(x,y):=\inf\{l(\gamma)\,:\, \gamma:[0,1]\to X\mbox{ continuous },\;\; \gamma(0)=x, \gamma(1)=y\},
 \end{equation} where $l(\gamma)$ denotes the length of $\gamma$. %(defined as $\sup \sum_{i=0} ^{k-1} |\gamma(t_i)-\gamma(t_{i+1})|$, where the supremum is taken over all the partitions $0=t_0<\dots<t_k=1$).
 When $x$ and $y$ do not belong to the same connected component, we put $d_X(x,y):=\infty$. We call $d_X(\cdot,\cdot)$ the {\bf inner metric of $X$}.

 We say that  $f:X\to\R$ is {\bf Lipschitz with respect to the inner metric} if there exists a constant $C$ such that the following inequality holds for any $x,y\in X$
 \begin{equation}\label{eq_inner_lips}|f(x)-f(y)|\le Cd_X(x,y).\end{equation}
 We denote by   $\check{\cc}^{0,1}(X)$ the space of the functions on $X$ that are Lipschitz with respect to the inner metric of $X$.	 We endow this space with its natural norm \begin{equation}\label{eq_co1_norme}
 	||u||_{\check{\cc}^{0,1}(X)}:=\sup_X |u(x)|+ \sup_{X\times X, x\ne y} \frac{|u(x)-u(y)|}{d_X(x,y)}.                               \end{equation}
 
 \begin{rem}\label{rem_inner}In \cite{livre}, the infimum (\ref{eq_inner_dist}) is taken on the $\cc^0$ subanalytic arcs only but this makes no difference \cite{lipsapprox}.
 	Moreover, if $E$ is a nowhere dense globally subanalytic subset of $X$ then it is enough to take this infimum   on the $\cc^0$  subanalytic arcs $\gamma$ that cut $E$ at finitely many points only (possibly $0$). This is because we can  approximate $\cc^0$  globally subanalytic arcs of $X$ by $\cc^0$  arcs (of arbitrary close length) meeting $  E$ in finitely many points. As a matter of fact, if $u\in\check{\cc}^{0,1}(X\setminus E) $ extends continuously on $E$ then the extension is inner Lipschitz on $X$.
% 	\begin{proof}  Indeed, if $E$ is nowhere dense, for every $x\in E$ and there is $a_x$ such that each for $s\in (0,a_x)$ there is $y\in X\setminus E$ such that $|x-y| =s$.   By Definable Choice, this means that there is a globally subanalytic mapping $w(x,s)$ such that  $|w(x,s) -x|=s$. It follows from  Puiseux Lemma with parameters \cite[Proposition]{livre} that there is  a partition of $E$ into globally subanalytic sets such that $w$ is $\cc^1$ on each element of this partition and since $\gamma$ is subanalytic, we can find $t_0=0 <\dots <1=t_k$ such that $\gamma(t)$ stays in one element of this partition when $t\in (t_i,t_{i+1}), i<k$. Then the arc    $\tilde\gamma(t):=w(\gamma(t) , s(t)) $, with $s(t)$ small positive $\cc^1$ function satisfying $s(t_i)=s(t_{i+1})=0$ only cuts   $E$ at $s(t_i)$   
% 	\end{proof}
 	  
%$$ 	d_X(x,y):=\inf\{l(\gamma)\,:\, \gamma:[0,1]\to X\mbox{ continuous },\;\; \gamma(0)=x, \gamma(1)=y, \gamma([0,1)\cap E  \mbox{ finite}\},$$
 	\end{rem}
 
  Theorem \ref{thm_mor} ensures that for $k$ large enough, every $u \in W^{k,2}_\mu (\omd)$ must coincide  almost everywhere with a function which is  Lipschitz with respect to the inner metric. In the sequel, we will always assume it to be equal   on  $X$ to such a function.
 Embedding (\ref{eq_embed}) amounts to say that  there is a constant $C$ such that we have for $k\ge 1$ large enough for   $u\in 	W_ { \mu}^{k,2}(\omd)$:
 \begin{equation}\label{eq_cont_mor}
 	||u||_{\check{\cc}^{0,1}(X)} \le C 	||u||_{W_ { \mu}^{k,2}(\omd)}.
 \end{equation}
 
 \begin{rem}\label{rem_compact} 
This embedding is induced by the embedding  given by Corollary $6.3$ of \cite{lprime} (see the proofs of Theorem $3.4$ and $3.5$ of \cite{platrub} for details). As the latter Corollary provides an embedding which is compact in the $L^\infty$ norm, so does Theorem \ref{thm_mor}. Section \ref{sect_cpct} gives details on this issue. \end{rem}

\subsection{A Poincar\'e inequality.}
If   $\{a_1,\dots, a_l\}$ is a finite subset  of $X$ meeting every connected component, we then have for   $u\in W^{k,2}_\mu(\omd)$ (with $k$ as in (\ref{eq_cont_mor})) 
\begin{equation}\label{eq_poinc}
	||u||_{L_ { \mu}^{2}(\omd)} \le C \sum_{1\le j\le l} |u(a_j)|	+ C	\sum_{1\le i\le k} || D^i u||_{L^2 _\mu (\omd)},
\end{equation}
for some constant $C$. 
This inequality actually follows from  the compactness of embedding (\ref{eq_embed})  (see Remark \ref{rem_compact} or Section \ref{sect_cpct}) up to a classical argument. Namely, if it failed, there would be a unit  sequence $u_m$ in $L_ { \mu}^{2}(\omd) $  for which the right-hand-side goes to zero. The compactness of the embedding would then ensure us that we can extract  a uniformly convergent subsequence. This subsequence would tend to a function $u$ that has $0$ derivative as distribution. By Proposition $2.5$ of \cite{platrub},  there is a $\mu$-negligible subset $B$ of $X$ such that $X\setminus B$ can be split into manifolds on which $\mu$ is absolutely continuous with respect to the canonical measure of the manifold. The function $u$ must thus be locally constant on these manifolds, which, since $\sum_{j=1}^l |u_m(a_j)|$ tends to zero and the $a_j$'s meet each connected component,  would contradict the fact that $u_m$ is a unit sequence in   $L_ { \mu}^{2}(\omd) $. 

Inequalities (\ref{eq_cont_mor}) and (\ref{eq_poinc}) will ensure well-posedness of our ridge problem in the Sobolev space. The latter inequality makes it possible to get rid of the $L^2_\mu$ norm in the penalty of the loss (see (\ref{eq_ridge_theoretical})). It is more convenient to work with the second term of the right-hand-side of (\ref{eq_poinc}) (called by the author the ``Sobolev energy'') than with the Sobolev norm because it does not penalize the constant function. It seems to be particularly relevant for classification problems, where the target is locally constant and thus has Sobolev energy $0$. 

 We end this section with   remarks about our setting that will not be needed for the next sections.
 \begin{rem}\label{rem_definissabilite}
 	 	The above results remain valid  with minor modifications in the proofs in the larger framework of polynomially bounded o-minimal structures \cite{cos,vdd}.
 	\end{rem}

 \begin{rem} In    (\ref{eq_cont_mor}) and (\ref{eq_poinc}), the function $u$ is defined on a neighborhood $\omd$ of $X$. However, only the values of the partial derivatives of $u$ on $X$ are relevant, and these two inequalities are valid for any $u$ defined on a (possibly non subanalytic) open neighborhood $V$ of $X$ in $\omd$, with constants independent of $V$. This is because  the restriction map $W^{k,p}_\mu(\omd)\to W^{k,p}_{\mu_{|V}} (V)$ is an onto isometry. This means that if we define $W^{k,p}_\mu(X)$ as the direct limit of $W^{k,p}_{\mu_{|V}}(V)$ when $V$ runs in all the open neighborhoods of $X$ in $\omd$, these two inequalities remain valid in this space. We thus could forget $\omd$, simply fixing a locally closed bounded subanalytic set $X$ and working with functions defined on any neighborhood of this set.  \end{rem}
 
 \begin{rem}\label{rem_equivalence} Two measures $\mu$ and $\nu$ are {\bf bi-Lipschitz equivalent} if there is a bi-Lipschitz homeomorphism $h:V\to \omd$, with $V$ and $\omd$ neighborhoods of the support of $\nu$ and $\mu$ respectively,  such that  $\frac{h_*\nu}{C}\le \mu \le Ch_*\nu$ for some $C>0$. It seems that the proof of Corollary 3.5 of \cite{platrub} applies to every measure  with bounded support which is bi-Lipschitz equivalent to a push-forward of a globally subanalytic measure on an open subset of $\R^n$, and hence that (\ref{eq_poinc}) is valid for such a measure.
  	 It indeed seems to be enough to have such a homeomorphism locally at each point of  the closure of the support in $\R^n$ (we can use a partition of unity).
 	In particular, our method apparently applies to every measure with a density bounded away from zero and infinity on a bounded   Lipschitz domain or   a compact Lipschitz submanifold (possibly with boundary\footnote{By a Lipschitz submanifold with boundary of $\R^n$ we mean a set which is locally, up to a linear change of coordinates of $\R^n$, the graph of a Lipschitz map over a Lipschitz half-space.}) of $\R^n$, which is a very common framework. 
%Inequalities (\ref{eq_cont_mor}) and (\ref{eq_poinc})  actually hold for any  measure  on a bounded open set which is bi-Lipschitz equivalent  to a push-forward of  globally subanalytic measure. Indeed, the proof of Theorem \ref{thm_mor} first splits the considered measure  into a finite sum $\sum_{M} \fbf_M\cdot \mu_M$ (see Proposition \ref{pro_push}) and then focuses on the spaces $W^{1,p}_{\fbf_M\cdot \mu_M}(M):=\{u\in L^p_{\fbf_M\cdot \mu_M}(M): |\nabla u | \in L^p_{\fbf_M\cdot \mu_M}(M)\}$.
%If $h: (V,\nu)\to (\omd,\mu)$ is a bi-Lipschitz equivalence with $\omd$ bounded subanalytic open set and $\mu$ push-forward of some globally subanalytic measure, then $(u\circ h^{-1})_{|M}\in W^{1,2}_{\fbf_M\cdot\mu_M}(M)$ for every $u\in\cc^{k,2}_\nu(V)$, $k\ge 1$. We thus can  apply Theorem 3.3 of \cite{platrub}  to this function, as well as to $(\frac{\pa^{|\alpha|}u}{\pa x^\alpha}\circ h^{-1})_{|M}\in W^{1,2}_{\fbf_M\cdot\mu_M}(M)$ for all $|\alpha|\le (k-1)$ and $M$, to conclude that $u\circ h^{-1}\in\check\cc^{0,1}(X)$ ($X$ being the support of $\mu$), hence $u\in\check\cc^{0,1}(h^{-1}(X))$, and that $ W^{k,2}_{\nu}(V)\hookrightarrow \check\cc^{0,1}(h^{-1}(X))$, for $k$ large enough (see the proof of \cite[Corollary 3.3]{platrub} for details).   Inequality (\ref{eq_poinc}) then follows by the same argument. Since we can use a finite partition of unity, it is enough to have such a homeomorphism locally at each point of $\adh V$ (assuming $V$ bounded).
%holds for $\nu$ with a constant independent of $u$, depending on the bi-Lipschitz constants of $h$ and on the above $C$.
\end{rem}
 
 	\begin{rem} 
 		If $E$ is a nowhere dense globally subanalytic  subset of $X$ then, applying 
 		  (\ref{eq_cont_mor}) to $\omd\setminus \adh E$   endowed with the restriction of $\mu$, we get for each $u\in W^{k,2}_\mu(\omd\setminus \adh E)$ (for $k$ large enough, for some $C>0$ independent of $u$) %that extends continuously at points of $\adh E$
$$	||u||_{\check{\cc}^{0,1}(X\setminus \adh E)} \le C 	||u||_{W_ { \mu}^{k,2}(\omd\setminus \adh E)}.$$ %	which is exactly (\ref{eq_cont_mor}) if $E$ is $\mu$-negligible.	  
When  $u$ extends continuously at points of $\adh E$,  we must have $	||u||_{\check{\cc}^{0,1}(X\setminus \adh E)} =			||u||_{\check{\cc}^{0,1}(X)} $ (see Remark \ref{rem_inner}).   This means that  (\ref{eq_cont_mor}) (and hence (\ref{eq_poinc})) holds for continuous functions that have singularities on a nowhere dense globally subanalytic set and  we can get rid of the singularities of $u$ when integrating. 
 	  This situation arises for instance when $u$ is the absolute value of a polynomial or more generally any piecewise polynomial continuous function on a globally subanalytic partition. 		
 		 We wish to stress the fact that the constants, and possibly $k$, however then depend on this subset $E$, since the partition  associated to the restriction of the measure (see Definition \ref{dfn_sub_meas} or Proposition \ref{pro_push}) must now be adapted to $X\setminus \adh E$.  For example, if $E_t$ denotes  the boundary of the ball $\bou(0_{\R^5},t)$ and $u_t(x)=\max (1-\frac{|x|^2}{t^2},0)$ then $||u_t||_{W^{k,2}_\mu((-1,1)^5\setminus E_t)}$ tends to $0$   as $t\to 0$ (for all $k$, say that $\mu$ is the Lebesgue measure) but $u_t$ does not tend to $0$ uniformly, although this function is smooth outside $E_t$ and continuous for each $t$.

% 	\end{enumerate}
 \end{rem}
 \subsection{On asymptotical values of elements of $\check{\cc}^{0,1}(X)$}\label{sect_cpct}  This section  is   not needed for the next sections. Elements of $\check{\cc}^{0,1}(X) $    do not necessarily extend continuously at frontier points of $  X$ ($X$ is closed in $\omd$ but may have accumulation points in $\adh\omd\setminus \omd$). We clarify this in this section, which will be the opportunity to account for the  compactness of our Morrey's embedding explained in Remark \ref{rem_compact}, as well as to point out that our method makes it possible to produce discontinuous heuristics (see example \ref{exa_disc}).

 Let $k$ be sufficiently big for   inequality (\ref{eq_cont_mor}) to hold.  We denote by  $\bou(x,\ep)$ the euclidean ball of radius $\ep$.
 We say that $X$ is {\bf connected at $\xo\in\adh X\setminus X $} if  $\bou(\xo,\ep)\cap X$ is connected for all  $\ep>0$ small enough.  If $X$ is connected at a frontier point $\xo$ then an inner Lipschitz function can have only one asymptotical value at $\xo$ (this easily follows from the Lipschitz conic structure theorem \cite[Theorem $3.4.1$]{livre}).  When the germ of $X$ at $\xo$ has more than one connected component (there must be finitely many), there then may be several asymptotical values at $\xo$, one per connected component (since we can argue on each component).

We say that $X$ is {\bf normal} if it is connected at each $x\in  \adh X\setminus X$. By the above paragraph, if $X$ is normal then $\check{\cc}^{0,1}(X)=\check{\cc}^{0,1}(\adh X)$ (see Remark \ref{rem_inner}). Since $\adh X$ is compact for $d_{\adh X}$, it then follows from Arzela-Ascoli's Theorem  that  every bounded sequence of $\check{\cc}^{0,1}(X)$ has a uniformly convergent subsequence.

When $X$ is not normal, one can always divide it into finitely many normal sets. It suffices to take a triangulation of $\adh X$ compatible with $X$. Applying then the above argument to every element of this partition, we get that every bounded sequence of $\check{\cc}^{0,1}(X)$ has a uniformly convergent subsequence, from which we can conclude that embedding (\ref{eq_embed})   is compact in the $L^\infty$ norm. 

%
%{\bf A $\cc^\infty$ normalization of a globally subanalytic submanifold $M$ of $ \R^n$ }
%is  a globally subanalytic $\cc^\infty$ diffeomorphism $h: \check{M}\to M$ satisfying $\sup_{x\in \check{M}} |D_x h|<\infty$ and  $\sup_{x\in M} |D_x h^{-1}|<\infty$,  with $\check{M}$ normal $\cc^\infty$ submanifold of $\R^m$, for some $m$. We  sometimes say that $\check{M}$ is a normalization of $M$.
%Existence of normalizations was established in \cite[Proposition $3.4$]{trace}.
%
%The sets $X$ and $\omd$ being globally subanalytic, $X$ admits a neighborhood $U$ in $\omd$ that  retracts by deformation on it (this follows from existence of triangulations). A normalization of $U$, say $h_U:\check{U} \to U$, thus provides a mapping $h:h_U^{-1}(X) \to  X$ with $\check{X}:=h^{-1}_U(X)$ normal.  
%
%
%
%As $h_U$ and $h^{-1}_U$ both have bounded derivative, $h$ is bi-Lipschitz with respect to the inner metric.  Hence, $h$ identifies $\check{\cc}^{0,1}(X)$ with $  \check{\cc}^{0,1}(\check{X})$, which, since $\check{X}$ is normal, is equal to $\check{\cc}^{0,1}(\adh{\check{X}})$. As $\adh{\check{X}}$ is compact,   thanks to Arzela-Ascoli's Theorem, we can conclude that (\ref{eq_embed}) actually yields a compact embedding. 

The argument we just explained  is indeed essentially the argument of Remark \ref{rem_compact} since the compactness of the embedding  given by Corollary $6.3$ of \cite{lprime}  is  obtained by a normalization of the manifold. 
\begin{exa}\label{exa_disc}
	Let $\omd =(-1,1)^3$ and let $X$ be the cone $$\{ (x,y,z)\in (-1,1)^3:  x^2+y^2=z^2 \}.$$
The set $X$ being normal, any element of $\check{\cc}^{0,1}(X)$ extends to $\adh X$.  This means that so do all the
 elements of $W^{k,2}_\mu(\omd)$ (for $k$ large enough, say for instance that $\mu$ is the Hausdorff measure on $X$ of dimension $\dim X$).  If now $\omd'=\{ (x,y,z)\in (-1,1)^3:  x^2+y^2<2z^2 \} $ and $X':= X\setminus \{0\}$ then the elements of $W^{k,2}_\mu(\omd')$ may have two asymptotical values at the origin. This  makes it possible to produce discontinuous heuristics, provided we include discontinuous functions in the basis $f_1,\dots,f_b$ (appearing in the next section), which  requires to   know in advance at which points the target function fails to be continuous.   %This also demands the support of the measure to have several connected components at those points.
	\end{exa}

 \section{Our greedy algorithm}\label{sect_algo}
 We are in the same setting as in the introduction. Namely, let $\omd$ be  an open bounded  
 globally subanalytic subset of $\R^n$ and let $\mu$ be a globally subanalytic probability measure on $\omd$ or a push-forward of such a measure under a globally subanalytic mapping. Let $\xca:=\{x_i\}_{i\in \N}$ be a sequence of observations drawn independently with respect to $\mu$ and  let $X$ denote the support of $\mu$. This set is with probability $1$ the closure of $\xca$ in $\omd$ (strictly speaking, the closure of  $\{a\in \R^n: \exists \, i\in \N, x_i=a\}$). Take an inner Lipschitz function $f:X\to \R$, and assume $x_1,\dots,  x_\ell$ to be respectively labeled by $y_1:=f(x_1),\dots,y_\ell:=f(x_\ell)$,  $\ell\in \N$ fixed integer.
 
 As often in machine learning, we will not strive to learn $f$ but a close function within our basis.  Our aim will be to find an element of the Sobolev space $u$ that approximates satisfyingly $f$.
  We emphasize that we do not assume $f$ to be in the Sobolev space.   
   We will be looking for a function $u$ in $ W^{k,2}_\mu(\omd)$, with $k$ sufficiently large for (\ref{eq_cont_mor}) to hold, satisfying  $|u(x_{i_j})-f(x_{i_j})|<\ep$ for all $j\le l$ ($\ep$ small positive number), where  $x_{i_1},\dots, x_{i_l}$ is a subset of the labeled data sufficiently rich to contain at least one point in each connected component of $X$ (the algorithm will  rather work with the MSE train, this is just to motivate our method).  In this situation, we have for all $x\in X$ 
 $$|u(x)-u(x^*)|\le C ||u||_{W_\mu^{k,2}(\omd)} \min_{j\le  l}  d_X (x,x_{i_j}) \et |f(x)-f(x^*)|\le L_f\,\min_{j\le  l}  d_X (x,x_{i_j}) ,$$
 where $L_f$ is the Lipschitz constant of $f$, $C$ the constant of (\ref{eq_cont_mor}), and $x^*$ the point realizing    $ \min_{j\le  l}  d_X (x,x_{i_j} )$. Summing these three inequalities, we get an upper bound of the error:
 \begin{equation}\label{eq_error_f_u}|f(x)-u(x)|\le\ep +(L_f+C ||u||_{W_\mu^{k,2}(\omd)})\min_{j\le  l}  d_X (x,x_{i_j}) .\end{equation}
On a subspace of finite dimension $H$, (\ref{eq_cont_mor}) always holds, but with a constant depending on $H$. Since the greedy procedure browses a large number of candidate spaces, what matters here is a constant independent of them, which is what (\ref{eq_cont_mor}) provides.
 
The Lipschitzness of $f$ could of course be relaxed into another notion of uniform continuity property. Existence of such a function $u$ is here an assumption but the algorithm will on the way check the feasibility by computing the MSE train of the output and providing bounds for its Sobolev energy.
 
  We will construct our approximations by solving a quadratic problem involving the MSE train and a penalty that will be given by a weighted Sobolev semi-norm. These weights will play the role of our regularization parameters. Take   some positive reals  $(\wbf_\alpha)_{1\le |\alpha|\le k} $, with $k$ as above and $\alpha=(\alpha_1,\dots, \alpha_n)$, and set for $u$ and $v$ in $W^{k,2}_\mu(\omd)$
  \begin{equation}\label{eq_inner_weighted_theoretical}
 	<u,v>_\wbf:= \sum_{1\le|\alpha|\le k}\wbf_\alpha\int_\omd  \frac{\pa ^{|\alpha|} u}{\pa x^\alpha}(x)\cdot \frac{\pa ^{|\alpha|} v}{\pa x^\alpha}(x) d\mu(x). \end{equation}
 Given a subspace $H$ of $W^{k,2}_\mu(\omd)$, consider now the ridge problem
 \begin{equation}\label{eq_ridge_theoretical}
 	\underset{u\in H}{\mbox{ argmin}} \;\frac{1}{\ell}\sum_{1\le i\le \ell}|y_i-u(x_i)|^2+||u||_\wbf^2,
 	 \end{equation}
  where $||u||_\wbf$ is the semi-norm resulting from the just above inner product. Inequality (\ref{eq_poinc})  ensures that this problem is well-posed for any closed subspace $H$ as soon as the $x_i$ meet every connected component of $X$.
%  
%  They are here to improve the minimizer, giving in our experiments less importance to higher derivatives, which with Gaussians may grow exponentially.  

  The subspace $H$ will be chosen as the span of some functions $f_1,\dots, f_b$ in $W^{k,2}_\mu(\omd)$.
  In the ideal situation where we are able to estimate inner products  $<f_i,f_j>_\wbf$,
 we can compute the Gram matrix $G$ in this basis and solve the above ridge problem. 
  This is unfortunately almost always impossible and we will rather estimate the inner product by averaging 
 the partial derivatives on  $N$ unlabeled data drawn independently of all the other choices (and independently of each other). Namely, let $\mathcal{Z}=(z_1,\dots,z_N)$ be $N$ random points drawn independently in $\xca$, and  define
 \begin{equation}\label{eq_<>_Z}<u,v>_{\mathcal{Z}}:= \frac{1}{N}\sum_{1\le |\alpha|\le k} \wbf_\alpha  \sum_{i=1}^N\frac{\pa ^{|\alpha|}u}{\pa x^\alpha}(z_i)\cdot \frac{\pa ^{|\alpha|}v}{\pa x^\alpha}(z_i). \end{equation} 
  This gives rise to an empirical norm $|| u||_{\mathcal{Z}}$ and an empirical loss
 \begin{equation}\label{eq_empirical_loss}
 	L_\mathcal{Z}(u):=	 \sum_{1\le i\le \ell} \frac{|y_i-u(x_i)|^2}{\ell}+	 ||u||_\zca^2.
 \end{equation}
In order to avoid  an ill-conditioned problem, we will restrict ourselves to a spectral subspace $E_{\zca,\tau}\subset \R^b$ spanned by all the eigenvectors corresponding to an eigenvalue  bigger than some threshold $\tau>0$ (fixed once for all at the beginning) of the matrix 
%\begin{equation}
$M:=	\frac{1}{\ell}A^\mathbf{t}A + G_\zca,$
%\end{equation} 
where $A$ is the matrix of the linear mapping $H \ni u \mapsto (u(x_1), \dots, u(x_\ell ))$ in the basis $f_1,\dots,f_b$ and $G_\zca$  denotes the  Gram matrix  of $<\cdot, \cdot>_\zca$ in this basis.  We shall therefore solve
\begin{equation}\label{eq_mz_Z_tau}
	\underset{\omega\in E_{\zca,\tau}}{\mbox{argmin}} \;\;\frac{1}{\ell} |\ybf-A\omega|^2 +\omega^{\mathbf{t}}  G_\zca\omega .
\end{equation} 
%Denote by $\omega$ the solution of this quadratic problem and by $u$ the corresponding element of $H$, i.e. $u=\sum \omega_i f_i$. 

 The error of our Monte Carlo approximation can be bounded with high probability  using Hoeffding's  inequality, which  gives for all $u\in H$ an estimate of type (see  (\ref{eq_coeffs}))
 \begin{equation}\label{eq_etat}
 	||u||_\wbf ^2\le ||u||_\zca^2+ \eta,
 \end{equation}
 for some $\eta$ depending on $\ybf, \wbf, N,\omega$, and the $f_i$. 
 McDiarmid's inequality  can be used to provide other inequalities of the same type (see (\ref{eq_sph}) and Remark \ref{rem_B}), that may be more or less sharp, depending on the situation. We can of course take the smallest $\eta$ of inequalities (\ref{eq_coeffs}) and  (\ref{eq_sph}) with probability at least $(1-2p)$, for $p\in (0,1)$ fixed (small). 
 
  Notice that  if $\omega\in \R^b$  is the minimizer of (\ref{eq_mz_Z_tau}) and if  $u=\sum \omega_i f_i$    then, because $ L_\zca(u)\le L_\zca(0)$, we must have
  $$  ||u||^2_\zca \le\frac{1}{\ell}  |\ybf|^2 -\frac{1}{\ell} |\ybf-A\omega|^2, $$
  and consequently, by (\ref{eq_etat}),  which is true for all $\omega$,
  \begin{equation}\label{eq_energy_minimizer}
  	||u||^2_\wbf \le\frac{1}{\ell}  |\ybf|^2 -\frac{1}{\ell} |\ybf-A\omega|^2 +\eta,
  \end{equation}
 which means that the Sobolev energy  is, up to $\eta$, bounded by the fall of the MSE train.
 
   The number of functions of the basis is generally limited by the cost of the computation of $G_\zca$, which  makes it difficult to approximate an arbitrary target. We therefore opted for a Greedy approach. Our algorithm can be described as follows:

\noindent   \begin{tcolorbox}[
	colback=white,
	colframe=black!70,
	boxrule=0.6pt,
	arc=2mm,
	left=3mm,
	right=3mm,
	top=2mm,
	bottom=2mm,
	title={Multiscale Greedy},
	fonttitle=\bfseries
	]
	Choose a basis of functions $f_1,\dots,f_b$  \\
	Draw randomly   $z_1,\dots, z_N $ in $\xca$\\
	Compute $<f_i,f_j>_\zca$ for all $i\le j$, where $\zca=(z_1,\dots, z_N)$\\    
	Find the solution $\omega$ of Problem (\ref{eq_mz_Z_tau})\\
	Compute the residual $\mathbf{r}:=\ybf -A\omega$, where $A$ is as in (\ref{eq_mz_Z_tau}), and iterate the method, with $\ybf:=\mathbf{r}$ and a new basis of functions. 	
\end{tcolorbox}
  
The  $z_i$ must be drawn independently of all the other choices. As explained in (\ref{eq_etat}), the theorems of Section \ref{sect_stat} provide at each level of the greedy a bound for the weighted Sobolev norm of the minimizer $u=\sum \omega_i f_i$. To make it more explicit, let $\mathbf{r}_l$ denote the residual at the level $l\ge 1$ (with $ \mathbf{r}_0:=\ybf$) and $u_l=\sum_{i=1}^b \omega_{l,i}f_{l,i}$ the function corresponding to the minimizer, and observe that (\ref{eq_energy_minimizer}) reads for $l\ge 1$:
$$||u_l||_\wbf^2 \le \frac{1}{\ell} |\mathbf{r}_{l-1}|^2-\frac{1}{\ell} | \mathbf{r}_l|^2 +\eta_l,$$
for some $\eta_l$ that can be bounded with high probability as explained in (\ref{eq_coeffs}) and (\ref{eq_sph}).  After summing over all the accomplished levels, we get with probability at least $(1-2lp)$:
%$$   \sum_{i=1}^l ||u_i||_\wbf^2 \le\frac{1}{\ell}|\ybf|^2 +\sum_{i=1}^l \eta_i.$$ \le \sqrt{l}\,\Big(\tfrac1\ell|\ybf|^2+\sum_{i=1}^l\eta_i\Big)^{1/2}
\begin{equation}\label{eq_aposteriori} ||\sum_{i=1}^l u_i||_\wbf \le \sum_{i=1}^l ||u_i||_\wbf \le \sum_{i=1}^l \Big(\frac{1}{\ell}|\mathbf{r}_{i-1}|^2-\frac{1}{\ell}|\mathbf{r}_i|^2+\eta_i\Big)^{1/2}.\end{equation}
Let us stress the fact that $\sum_{i=1}^l \Big(\frac{1}{\ell}|\mathbf{r}_{i-1}|^2-\frac{1}{\ell}|\mathbf{r}_i|^2\Big)\le\frac{|\ybf|^2}{\ell}$.
  The  $\eta_i$ thus turn out to be bounds for the total weighted Sobolev energy of the heuristics rather than limitations of the Monte Carlo approximation of the Gram, which means that keeping reasonable values of the $\eta_i$ is essential.  Inequalities   (\ref{eq_coeffs}) and (\ref{eq_sph}) provide effective values of the $\eta_i$ in terms of the hyperparameters, the minimizer $\omega$, and the probability $p$. It is therefore possible to provide, after each level of the greedy   algorithm, an a posteriori bound for the weighted Sobolev energy of the produced heuristic valid with high probability.

 In view of (\ref{eq_error_f_u}) and (\ref{eq_aposteriori}), attention should be taken to prevent the $\eta_i$ from exploding. Of course, they depend on the choice of the successive bases $f_{l,i}$.  Their   ability to adapt to the residuals (to which the $\sqrt{\eta_i}$ are proportional)  is thus essential for the success of the algorithm. A natural choice is to start with Gaussians with large radii and, at each level, to make the radius decay by a factor $q$. The first levels approximate the target roughly by rather mild functions, capturing its global behavior at low cost, and the next levels focus on more local oscillations. As the constants $\eta_i$ given by  Section \ref{sect_stat} will increase as the radius $\sigma$ of the Gaussians drops,  it is essential that the norm of the residual decays fast. We thus selected our basis among a large set of functions, a dictionary, taking those that are the most correlated with the residual. It seems that improvements could be made on the way this selection is made and that  the residual is a good indicator for this purpose. The  $\eta_i$ can be helpful  to ensure that a reasonable approximation of the target lies within the limit of our sequence of dictionaries and could make it possible to adapt  them to the target.

 The $\eta_i$ were computed on examples (targeting a probability $(1-p)$ with $p\le 10^{-3}$, see Section \ref{sec_tests}) and 
       stayed contained in the last levels by the rapid decay of the residual. It seems that many other statistical bounds can be provided and this is another interesting issue. The actual value of the Sobolev energy was also assessed empirically (on the total cloud) and the energy, as well as its empirical estimate, were much smaller than $\eta_i$, which means that there is space for improvements of the statistical bounds, for instance by taking more care of the overlaps of the bumps of the elements of the basis. 
       % On the last levels, the bumps of the Gaussians are separated from each other and the energy can 

\subsection{Directional derivatives.}\label{sect_dir}  The number of partial derivatives of order $ d$ involved in the computation of the Gram matrix equals  $\binom{n+d-1}{n-1}$, which is a serious problem  when the dimension of the ambient space increases, the computation of this matrix being the most expensive part of the algorithm.   With MNIST data downsampled into   $7\times7$, this was intractable for us. We thus used directional derivatives $$\frac{\pa^i u}{\pa w^i}(x):=D^i_x u(w,\dots,w)$$ to compute $G_\zca$, averaging at each  $z_i$ over a bunch  $\mathcal{B}_i$ of $m$ random vectors of $\sph^{n-1}$, i.e. we replaced the bilinear form $<\cdot,\cdot>_\zca$ introduced in (\ref{eq_<>_Z})  with 
  $$<u,v>_{\mathcal{Z}}:= \sum_{j=1}^k\, \frac{\wbf_j}{N} \, \sum_{i=1}^N\, \frac{1}{m}\,\sum_{w\in \mathcal{B}_i}\,\frac{\pa ^{j}u}{\pa w^j}(z_i)\cdot \frac{\pa ^{j }v}{\pa w^j}(z_i), $$
   for some weights $(\wbf_j)_{1\le j \le k}$.
   This amounts  to replace the bi-linear form (\ref{eq_inner_weighted_theoretical})   with 
    $$<u,v>_{\wbf}:= \sum_{j=1}^k  \wbf_j \int_\omd \int_{\sph^{n-1}} D_x^j u(w)\cdot D ^j _xv(w)\,dw\,d\mu(x),$$
  which  gives rise to an equivalent semi-norm. This proved efficient and the algorithm worked with a rather small number of directions ($50$ in dimension $49$), and increasing the number of directions did not induce significant differences in the results. Somehow, we here take advantage of the Monte Carlo architecture  which seems to compensate the just above Monte Carlo approximation.  It would be interesting to also give statistical bounds for empirical estimates based on directional derivatives.
  
  \subsection{Statistical bounds for the Gram estimate.}\label{sect_stat}
   In this section, we provide inequalities to bound the Sobolev energy in terms of the hyperparameters and our Monte Carlo estimate. We let $f_1,\dots, f_b$, $\omd, \mu, X, \wbf_\alpha,G,\zca,G_\zca, ||\cdot||_\zca$, and $\ybf$ be as above. We assume that the $f_i$ are $\cc^k$ and have bounded derivatives  up to order $k$  on $X$.   Given a matrix $U$, we denote by $||U||$ the norm provided by the supremum on the sphere  (w.r.t. the euclidean norm) and the $L^\infty$ norm of matrix valued mappings will be assumed to be taken with respect to this matrix norm.
   \begin{thm}\label{thm_coefs}
   	Let $p\in (0,1)$.	  With probability at least $(1-p)$, we have for every   $u=\sum \omega_i f_i$, $\omega\in \R^b$,
   	\begin{equation}\label{eq_coeffs}||u||_\wbf^2 \le ||u||^2_\zca+\Big(\sum_{j=1}^b|\omega_j|\cdot d_j\Big)^2\sqrt{\frac{2\ln(b(b+1)/p)}{N}}, \end{equation}
   	where $d_j=||\sum_{\alpha}\wbf_\alpha   \frac{\pa ^{|\alpha|}f_j}{\pa x^\alpha}(x)^2||_{L^\infty_\mu(\omd)}^{1/2}$.
   \end{thm}
   \begin{proof}
   	%		\begin{eqnarray}\label{eq_Gij}\sum_{\alpha}\wbf_\alpha       \frac{\pa ^{|\alpha|}f_i}{\pa x^\alpha}(x)\cdot  \frac{\pa ^{|\alpha|}f_j}{\pa x^\alpha}(x) 
   		%	 & \le&   \left(\sum_{\alpha}\wbf_\alpha   \frac{\pa ^{|\alpha|}f_j}{\pa x^\alpha}(x)^2\right)^\frac{1}{2}    \left(\sum_{\alpha}\wbf_\alpha \frac{\pa ^{|\alpha|}f_j}{\pa x^\alpha}(x)^2 \right)^\frac{1}{2}\le d_id_j
   		%	\end{eqnarray}
   	Let (the case $d_j=0$ is left to the reader) $$g_{ ij}(x):=\frac{1}{d_id_j}\sum_{1\le |\alpha|\le k}\wbf_\alpha\left(   \frac{\pa ^{|\alpha|}f_i}{\pa x^\alpha}(x)\cdot \frac{\pa ^{|\alpha|}f_j}{\pa x^\alpha}(x) \right),$$
   	and notice that
   	\begin{eqnarray*}|g_{ ij}(x) |
   		& \le& \frac{1}{d_id_j} \left(\sum_{1\le |\alpha|\le k}\wbf_\alpha   \frac{\pa ^{|\alpha|}f_i}{\pa x^\alpha}(x)^2\right)^\frac{1}{2}    \left(\sum_{1\le |\alpha|\le k}\wbf_\alpha \frac{\pa ^{|\alpha|}f_j}{\pa x^\alpha}(x)^2 \right)^\frac{1}{2} 
   		\le  1.
   	\end{eqnarray*}
   	Applying Hoeffding's inequality to the random variables $X_{ij,1},\dots , X_{ij,N}$ defined as 
   	$X_{ij,\iota}:=\frac{1}{2N} g_{ij}(z_\iota)\in [-\frac{1}{2N},\frac{1}{2N}]$ we get for each $i\le j\le b$ 
   	$$\mathbb{P}(|\sum_{1\le \iota\le N} X_{ij,\iota}-\mathbb{E}\sum_{1\le \iota\le N} X_{ij,\iota}|\ge \ep) \le 2\exp\left( -2\ep^2N\right).$$
   	Hence, for $p\in (0,1)$, cumulating these probabilities,  we get that  with probability at least $ (1-p)$ all the absolute value of the coefficients of the matrix $(G_\zca-G)$ are smaller than $2d_id_j\sqrt{\frac{\ln (b(b+1)/p)}{2N}}$. As a matter of fact, we have with probability at least $(1-p)$ for all $\omega\in \R^b$:
   	$$| \omega^\mathbf{t}(G_\zca-G)\omega| \le \sum_{i,j} 2d_id_j|\omega_i|\cdot|\omega_j|  \sqrt{\frac{\ln (b(b+1)/p)}{2N}}=  \left(\sum_j  d_j |\omega_j|  \right)^2\sqrt{\frac{2\ln (b(b+1)/p)}{N}},$$
   	which yields the desired bound.
   	%	\begin{eqnarray*}u(x)^2 
   		%		&=&  \sum_{i,j}\omega_i  \omega_j  \sum_{\alpha}\wbf_\alpha\left(   \frac{\pa ^{|\alpha|}f_i}{\pa x^\alpha}(x)\cdot \frac{\pa ^{|\alpha|}f_j}{\pa x^\alpha}(x) \right)
   		%		 \\& \le& \sum_{i,j} |\omega_i |\cdot | \omega_j  |  \left(\sum_{\alpha}\wbf_\alpha   \frac{\pa ^{|\alpha|}f_j}{\pa x^\alpha}(x)^2\right)^\frac{1}{2}    \left(\sum_{\alpha}\wbf_\alpha \frac{\pa ^{|\alpha|}f_j}{\pa x^\alpha}(x)^2 \right)^\frac{1}{2} 
   		%		  \\& \le& \sum_{i,j} |\omega_i |\cdot | \omega_j  | d_id_j
   		%		   \\& \le& \left( \sum_{i} |\omega_j  | d_i\right)^2
   		% \end{eqnarray*}
   \end{proof}
   The bound provided by the above theorem involves a term $O(\sqrt{\ln b})$ because we estimated the error on each coefficient in order to estimate the error on the matrices. In the proof of the theorem below, we rely on McDiarmid's inequality so as to argue on the level of matrices and get rid of this factor, which provides another bound (not comparable with the preceding one).

\begin{thm}\label{thm_sph}
	Let $G_x$    denote the  Gram matrix at $x\in \omd$, i.e. $$G_{x,ij}:=\sum_{\alpha}\wbf_\alpha   \frac{\pa ^{|\alpha|}f_i}{\pa x^\alpha}(x)\cdot \frac{\pa ^{|\alpha|}f_j}{\pa x^\alpha}(x).$$
	With probability at least $(1-p)$, $p\in (0,1)$, we have  for all $u=\sum_{i=1}^b \omega_i f_i, \omega\in \R^b,$
	\begin{equation}\label{eq_sph}
	||u||^2_{\wbf}\le ||u||^2_\zca+\eta\, |\omega| ^2,
	\end{equation}
where 
$$\eta = \sqrt{\frac{||\tra (G_x)||_{L^\infty_\mu(\omd)}||\tra (G_x)||_{L^1_\mu(\omd)} }{N}}+||\tra (G_x)||_{L^\infty_\mu(\omd)}\sqrt{\frac{\ln \frac{1}{p} }{2N}}.$$
	\end{thm}
\begin{proof}  
 Define a function $F: X^{N} \to \R$ by setting for  $\zca= (z_1,\dots,z_{N})\in X^{N}  $
	 $$F(\zca):= \lambda_{max}({G}-{G}_\zca),$$
	 where $\lambda_{max}$ denotes the biggest eigenvalue of this matrix. Note that, in virtue of Weyl's inequality, we have for any $i$, if $\zca=(z_1,\dots,z_{N})$ and $\zca'=(z_1,\dots, z_{i-1},z_i',z_{i+1}, \dots,z_{N})$
	 $$|F(\zca) -F(\zca') |\le \frac{ ||{G}_{z_i}- {G}_{z_i'}||}{N}, $$
	 so that, since $G_{z_i}$ and $G_{z'_i}$ are both positive symmetric matrices ($||A||=\sup_{|v|=1} |<v,Av>|$ for all $A$ symmetric):
	 $$|F(\zca)-F(\zca')|\le \frac{\max(||G_{z_i}||,||G_{z_i'}|| )}{N} \le  \frac{||{G}_x||_{L^\infty_\mu (\omd)}}{N}\le  \frac{||\tra \,{G}_x||_{L^\infty_\mu (\omd)}}{N}.$$

	 Applying McDiarmid's inequality to $F$ and  $\zca$ (now regarded as a finite family of random variables), we thus get for every $s>0$:
	 \begin{equation}\label{eq_mac}\mathbb{P}(F(\zca)\ge \mathbb{E}(F)+s)\le \exp\left(\frac{-2Ns^2}{||\tra \,G_x||^2_{L^\infty_\mu (\omd)} }\right). \end{equation}
	 Notice that, denoting by $<U,V>_{Fr}$ the Frobenius scalar product $\tra (U^\mathbf{t} V)$ and by $||.||_{Fr}$ the derived norm, we have: 
	 $$\mathbb{E} (F )\le \mathbb{E}(||{G}-{G}_\zca||_{Fr})
	  \le \sqrt{\frac{\mathbb{E}(||\sum_{i\le N }(\mathbb{E}({G}_x)-{G}_{z_i})||_{Fr}^2)}{N^2 }}= \sqrt{\frac{\mathbb{E}(||\mathbb{E}({G}_x)-{G}_{z_1}||_{Fr}^2)}{N }} $$
	% = \frac{\sqrt{\mathbb{E}(\sum_{z\in \zca}||\mathbb{E}({G}_x)-{G}_z||_{Fr}^2})}{N } 
	 since the $z_i$ are i.i.d. random variables and $(\mathbb{E}({G}_x)-{G}_{z_i}) $ has mean zero for all $i$ (the first inequality is from the definition of the Frobenius norm and the second is Jensen's inequality). But then 
	 $$ \sqrt{\frac{\mathbb{E}(||\mathbb{E}({G}_x)-{G}_{z_1}||_{Fr}^2)}{N }} = \sqrt{\frac{\mathbb{E}(||{G}_x||_{Fr}^2-||{G}||_{Fr}^2)}{N}}\le \sqrt{\frac{\mathbb{E}(||{G}_x||_{Fr}^2)}{N}}= \sqrt{\frac{\mathbb{E}(\tra\, {G}_x^2)}{N}},$$
	 from which we can conclude (using that $\tra A^2 \le  \lambda_{max}(A)\cdot \tra\, A$ for all $A\succeq 0$),
	  \begin{equation*}
	  	\mathbb{E} (F )\le \sqrt{\frac{||\lambda_{max}({G_x})||_{L^\infty_\mu(\omd)}\cdot \mathbb{E}( \tra\,{G}_x) }{N}}\le 
	  	\sqrt{\frac{||\tra\, {G_x}||_{L^\infty_\mu(\omd)}\cdot \mathbb{E}( \tra\,{G}_x) }{N}},
	  \end{equation*}
  since ${G}_x\succeq 0$. Now, applying (\ref{eq_mac}) with $$s:=||\tra\, {G_x}||_{L^\infty_\mu(\omd)}\sqrt{\frac{\ln 1/p}{2N}}, \quad p\in (0,1),$$ 
  and $\eta$ as in the statement of the theorem, we get that with probability at least $(1-p)$:
  $$F(\zca)\le \sqrt{\frac{||\tra\, {G_x}||_{L^\infty_\mu(\omd)}\cdot \mathbb{E}( \tra\,{G}_x) }{N}}+s. $$
  But since for all $\omega\in \R^b$
  $$||u||_\wbf^2-||u||_\zca^2 = \omega^\mathbf{t} (G-G_\zca)\omega \le |\omega |^2 F(\zca),$$
  this yields the desired estimate.
\end{proof}

%An interesting case is $B:=\frac{1}{\ell}A^\mathbf{t}A+\tau Id_{\R^b}$ where $\tau$ is the threshold appearing in (\ref{eq_mz_Z_tau}), which leads to  
%\begin{equation*}\omega^\mathbf{t}  B\omega \le \frac{1}{\ell} |A\omega|^2+ \tau |\omega|^2 \le   \frac{1}{\ell} |A\omega|^2+  \omega^\mathbf{t}M\omega\end{equation*}
% In the case where $\omega$ is the solution of (\ref{eq_mz_Z_tau}), we have thanks to minimality, $|A\omega|\le |\ybf|$ and $\omega M\omega \le  \frac{1}{\ell} |A\omega|^2$ so that we get for this $\omega$: 
% \begin{equation*}\label{eq_sph}\omega^\mathbf{t}  B\omega  \le \frac{2|\ybf|^2}{\ell}.\end{equation*}
% and consequently by (\ref{}) \begin{equation}
% 	||u||^2_{\wbf}\le ||u||^2_\zca+\eta\frac{2|\ybf|^2}{\ell}.,
% \end{equation}
\begin{rem}\label{rem_B}
The following generalization of the above theorem provides other (incomparable) statistical bounds.	Let $B$  be a symmetric positive  definite $b\times b$ matrix	 and set $$D_x:=B^{-1/2}G_xB^{-1/2}.$$
		With probability at least $(1-p)$, $p\in (0,1)$, we have  for all $u=\sum_i \omega_if_i$, $\omega\in \R^b$,
	\begin{equation*}
		||u||^2_{\wbf}\le ||u||^2_\zca+\eta\, \omega ^\mathbf{t}B \omega,
	\end{equation*}
	where 
	$$\eta = \sqrt{\frac{||\tra (D_x)||_{L^\infty_\mu(\omd)}||\tra (D_x)||_{L^1_\mu(\omd)} }{N}}+||\tra (D_x)||_{L^\infty_\mu(\omd)}\sqrt{\frac{\ln \frac{1}{p} }{2N}}.$$
	In particular, for $B=\frac{1}{\ell}A^\mathbf{t}A +\tau Id$, where $A$ is as in (\ref{eq_mz_Z_tau}), this gives a bound that takes in consideration the threshold $\tau$.   
	\end{rem}

\section{Numerical experiments}\label{sec_tests}
	\subsection{Regression problems on real algebraic varieties.
	}
	We   focused on algebraic hypersurfaces of $\R^4$, starting with the semi-algebraic set
	$$ X:=\{ (x,y,z,t)\in [0,1]^4: y^2+z^2+t^2=x^4 \mbox{ or } y^2+z^2+t^2=2x^4\mbox{ or }t+z+2y=2x \},$$
	fixing a sample of $4000$ points on this set (or very close) for the training phase and a test sample of $5000$ points. The clouds were generated by $60$ steps of gradient descent of random points. Such points are distributed according to the push-forward under a polynomial map, which corresponds to the setting of this article. Figure \ref{fig:2cones} shows a representation  in $\R^3$   (the variable $t$ has been dropped for this purpose but this does not affect the main properties of the singular point; the experiments use $X\subset[0,1]^4$).  	 This is typically an example of semi-algebraic set for which the inner metric is not equivalent to the euclidean metric.	
	\begin{figure}[ht]
		\centering
		\includegraphics[width=0.75\textwidth]{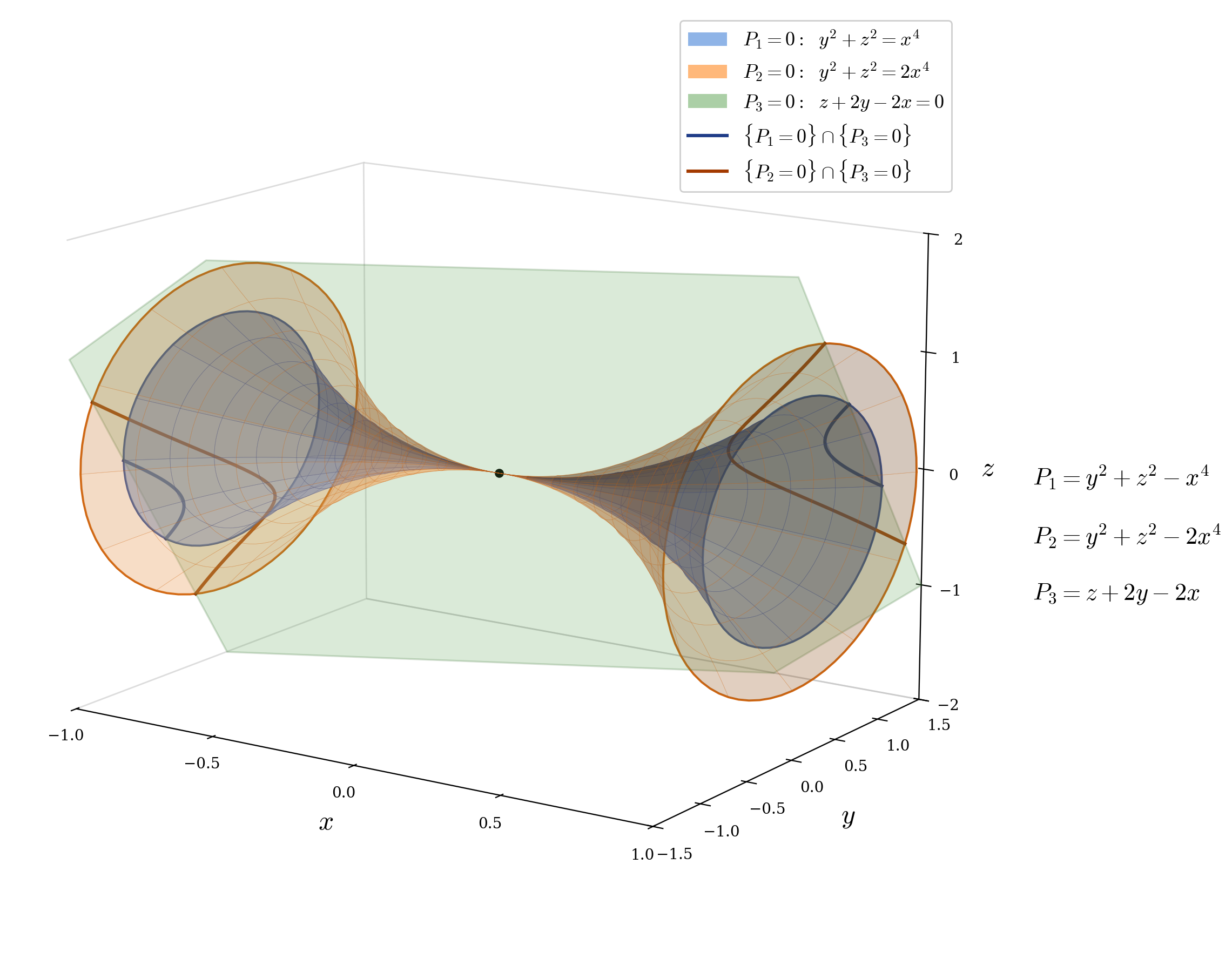}
		\caption{Three-dimensional visualization of the singular algebraic hypersurface defined by $ P_1 \cdot P_2\cdot P_3 =0$ (the set $X$ is its intersection with $[0,1]^4$).}
		\label{fig:2cones}
	\end{figure}

	 The target function was taken equal to
	\begin{equation}\label{eq_target}|Q_1(x,y,z,t)|+ |Q_2(x,y,z,t)|, \end{equation}
	with $$Q_1(x,y,z,t)=z^3 + ty^2 - 2zx^2
 \et Q_2(x,y,z,t)=t^3 - zyx + 2tx^2.$$ The absolute value is here to ensure a certain lack of smoothness of the target, so as to make the problem a bit more difficult.   We applied the greedy algorithm described in Section \ref{sect_algo}, solving the ridge problem (\ref{eq_mz_Z_tau}) on $12$ levels. The $i^{th}$-level dictionary  was constituted by $4000$ Gaussians, centered at the points of the clouds (labeled or not),   the radii of the Gaussians  being allowed to go from $1.5. (0.8)^{i-1}$  to  $1.5. (0.8)^{i+2}$, and we selected the $800$ functions that are the most correlated with the residual train (the Gaussians being rather similar functions, we did not take any normalization before estimating the correlation). The weights of the partial derivatives  of order $1\le j\le 3$  of the (squared) Sobolev norm were chosen equal to $10^{-4-j}$ (the other weights chosen as $0$) which proved sufficient. No effort was made to compute the minimal number of derivatives to ensure a well-posed problem on the entire Sobolev space.  The Gram matrix was computed using a random sample $1000$ of points (possibly unlabeled, this sample was different at each level) and  we set $\tau= 10^{-6}$ (see (\ref{eq_mz_Z_tau}) for $\tau$).  We carried out $20$ runs, varying the number of labeled points $100,200, \dots, 2000$, making no parameter tuning  between the runs.  Figure \ref{fig:graphes} displays the variation of the MSE Test and Train as functions of the number of labeled examples. As one could expect, the result decays as the labeled data are growing and the two graphs are getting closer.  During the greedy procedure, the MSE test decreased by up to 93 percent (by about 90 percent from 800 labeled points on).
  For comparison, we computed the MSE test of RBF Kernel Ridge Regression optimized with Grid Search, and our results are from 7 to 51 percent better (above 20 percent up to 1400 labeled points).

   An interesting point  of our test is that the MSE train and test were very correlated during the first levels, making it possible to estimate the potential MSE test gain of the last levels. The MSE train however started to decay faster than the MSE test on the last levels.  These levels only generating very low gain, an early stopping was added to interrupt the greedy algorithm when the MSE train is less than $5.10^{-5}$ so as to avoid overfitting and save time.

  \begin{figure}[ht]
  	\centering
  	\includegraphics[width=\textwidth]{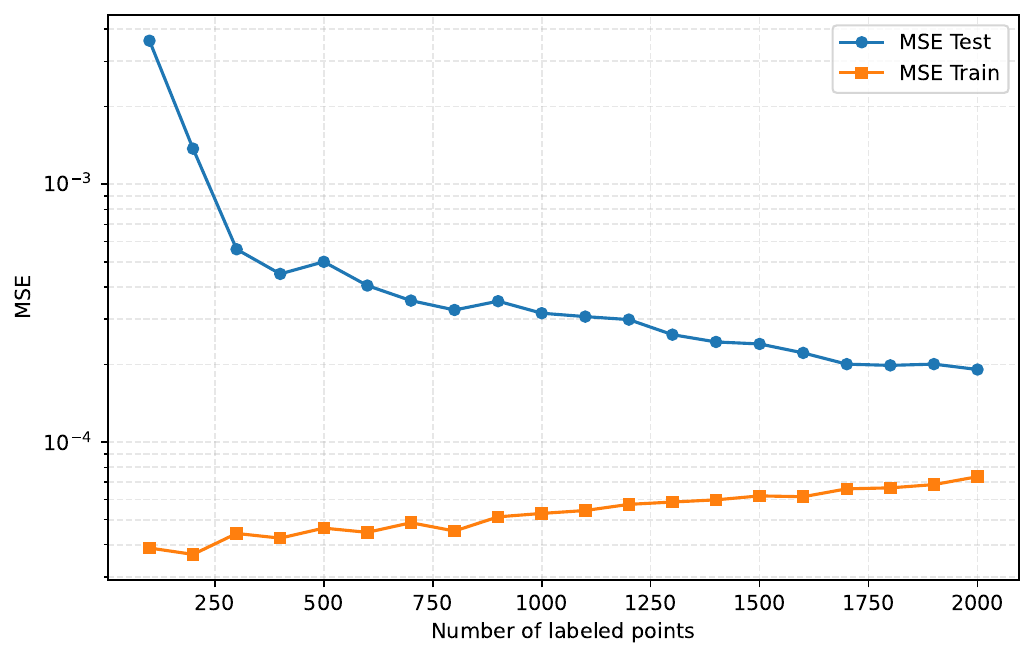}
  	\caption{MSE on the training and test sets as a function of the number of training samples on a logarithmic scale. %The training MSE is shown on a logarithmic scale, while the test MSE is shown on a linear scale.
  	}
  	\label{fig:graphes}
  \end{figure}
  
  We then studied random generalizations of this problem, taking the support of the cloud to be supported by a set $$X:=\{x\in [0,1]^4: P_1(x)\cdot P_2(x)\cdot P_3(x)=0 \},$$
  where the $P_i$ are random polynomials of degree $3$ vanishing at the origin (normalized by  the  square root of the mean  of squares of their coefficients and limited to 20 coefficients). The polynomials $Q_1$ and $Q_2$ were now   of degree $4$ and generated like the $P_i$ (random, limited to 20 coefficients, and normalized).  To keep the targets of the different random seeds within a comparable range, the absolute values were truncated at $1$, i.e. the target was $\min(|Q_1|,1) + \min(|Q_2|,1) $ (the same target was of course used for kernel ridge regression). 
  
  We made up to 18 levels.  The radii of the Gaussians of the level $i$ was however set equal to $1.5. (0.8)^i$ and the sample $\zca$ was composed by 2000 points, with all the other parameters as in the above particular case.
  The number of labeled points was fixed to $600$, and we only changed the random seeds (responsible for the cloud, the $P_i$, and the $Q_i$) at each run. Here we stress the fact that although the hyperparameters were set manually, there was no parameter tuning on each random seed and no validation was needed.   The average MSE test was about $4.8\cdot10^{-3}$ (with standard deviation $3.7\cdot10^{-3}$). For comparison, the average MSE test of ridge RBF on the $10$ runs amounted to $5.3\cdot10^{-3}$. Eight out of the ten runs gave results $5$ to $20$ percent better than the MSE test of RBF Kernel Ridge Regression optimized with Grid Search, while $2$ of them were less good.
	
	To be fully valid, the statistical bounds of Section \ref{sect_stat} require the sample $\zca$ to be drawn independently of the basis of functions. Since part  of the centers of the Gaussians are taken in the labeled data, this forces to exclude these data when drawing $\zca$, which was not done in the above tests. These statistical bounds were thus estimated on extra runs, assuming $p=10^{-3}$ and drawing $\zca$ as fresh samples, independent of the cloud, and, as mentioned in the preceding section, stayed contained by the rapid decay of the residual, although still being much higher than the actual Sobolev energy of the minimizer (evaluated empirically on a $20000$ points random sample).

  	\subsection{Application to classification problems.} We then applied our method to a classification problem  3,5 vs 8 on MNIST data downsampled  into $7\times7$. As above, the test was carried out on a sample of $4000$ examples (labeled and unlabeled), from $10$ different random seeds, and the results were  averaged.  We successively assumed  500, 1000, 1500, 2000 elements of this sample to be labeled. As explained in Section \ref{sect_dir}, we used the products of directional derivatives which were averaged among 50 random directions. The results of our multi-isotropic greedy algorithm (MISO) appear in Table 1. Kernel ridge regression is included as a reference; the accuracies obtained are of the same order.   Obtaining satisfying results actually demanded to raise the weights of order $j$ to $2.10^{1-j}$, $j=1,2,3$, as well as the radius and eigenvalue threshold (adjusted manually but identical for all the seeds), likely due to the high ambient dimension which increases distances.   The architecture was as described above, the dictionaries being constituted by 4000 Gaussian functions centered at sample points, from which we extracted 800 candidates at each level as above. The radius of the Gaussians was 12 at the first level decaying by 20 percent at every of the 10 levels. The accuracy at the first level was between $91.6\%$ and $94.3\%$ to finish as displayed in the table.  A feature of the method is that the norm is independent of the chosen dictionary, which made it possible to test other functions, like anisotropic Gaussians, Wendlands, or polynomial functions. We chose to work with isotropic Gaussian for the present tests because this lead to the most satisfying results. Adding a first level exploring polynomials of degree 2 (and then Gaussian levels) however provided a very slight gain ($0.2-0.3 \%$) and it seems that other dictionaries would deserve to be tested.

	\begin{table}[ht]
		\centering
	%	\caption{Performance on MNIST $3,5$ vs.\ $8$ over 10 random seeds.}
%		\label{tab:mnist}
		\begin{tabular}{c|cc|cc}
			$n_{\mathrm{train}}$
			& \multicolumn{2}{c|}{Accuracy}
			& \multicolumn{2}{c}{AUC} \\
			& MISO & RBF ridge & MISO & RBF ridge \\
			\hline
			500
			& $95.2 \pm 0.3$
			& $94.6 \pm 0.3$
			& $0.988 \pm 0.002 $
			& $0.985 \pm 0.002$ \\
			1000
			& $96.2 \pm 0.3 $
			& $95.9 \pm 0.3$
			& $0.993 \pm 0.001 $
			& $0.991 \pm 0.001$ \\
			1500
			& $96.7 \pm 0.3$
		& $96.5 \pm 0.2$
		& $0.994 \pm 0.001$
		& $0.993 \pm 0.001$ \\
			2000
				& $96.9\pm 0.3$
			& $96.7 \pm 0.2$
			& $0.995 \pm 0.001$
			& $0.994 \pm 0.001$ \\
		\end{tabular}
	\caption{Accuracy rate and AUC on MNIST $3,5$ vs.\ $8$, averaged over
		10 random seeds. Values are reported as mean $\pm$ standard deviation.}%;	brackets give the signed deviations of the minimum and maximum valuesfrom the mean

	\end{table}

\section*{Acknowledgments}
%The author thanks [Name(s)] for helpful discussions and comments on an
%earlier version of this manuscript, and the anonymous referee(s) for
%their careful reading and suggestions, which improved the presentation
%of the paper.

This work was supported by NCN funder grant 2021/43/B/ST1/02359.
During the preparation of this work, the author used  Claude (Anthropic), ChatGPT (OpenAI), Grok (X),
 AI language models, to improve the language and readability of the
manuscript, implement algorithms in Python, check computations, suggest reformulations of some
arguments.  The statement and proof of Theorem \ref{thm_sph} were
initially proposed by Claude, the rest of the mathematical content is due to the author. All tests were run on Google Colab. All content  
was  reviewed and edited  by the author, who  takes full responsibility for the content of
this publication.

\end{document}